\documentclass[12pt]{article}
\usepackage{latexsym}
\usepackage[all]{xy}

\usepackage{amssymb} 
 \usepackage{amsfonts}
\usepackage{algorithm}
 \usepackage{algorithmic}
\usepackage{amsthm}

\usepackage{graphicx}
\usepackage{xcolor}
\newtheorem{theorem}{Theorem}
\newtheorem*{theorem*}{Theorem}

\newtheorem{lemma}{Lemma}

\newtheorem{corollary}{Corollary}

\usepackage{amsmath}
\usepackage{amsfonts}
\newtheorem{proposition}{Proposition}
 
\usepackage{microtype}

\author{David Cohen-Steiner \thanks{UCA INRIA, {\tt david.cohen-steiner@inria.fr}} \and Alba Chiara De Vitis \thanks{UCA INRIA, { \tt alba.de-vitis@inria.fr}} }

\title{Spectral Properties of Radial Kernels and Clustering in High Dimensions\footnote{This work was partially supported by the Advanced Grant of the European Research Council GUDHI (Geometric Understanding in High Dimensions).}}

\begin{document}

\maketitle

\begin{abstract}
In this paper, we study the spectrum and the eigenvectors of radial
kernels for mixtures of distributions in $\mathbb{R}^n$. Our approach
focuses on high dimensions and relies solely on the concentration
properties of the components in the mixture. We give several results
describing of the structure of kernel matrices for a sample drawn from
such a mixture. Based on these results, we analyze the ability of
kernel PCA to cluster high dimensional mixtures. In particular, we
exhibit a specific kernel leading to a simple spectral algorithm for
clustering mixtures with possibly common means but different
covariance matrices. We show that the minimum angular separation
between the covariance matrices that is required for the algorithm to
succeed tends to $0$ as $n$ goes to infinity.

\end{abstract}

\section{Introduction}
\label{intro}

Given a set of data points drawn from a mixture of distributions, a basic problem in data analysis is to cluster the observations according to the component they belong to. For this to be possible, it is clearly necessary to impose separation conditions between the different components in the mixture.

Many approaches have been proposed to solve the problem of clustering mixtures of distributions. We give below a brief historical account of the algorithms that come with theoretical guarantees, focusing on the high dimensional situation. Unlike in the low dimensional case, approaches based {\it e.g.} on single linkage or spectral clustering cannot be employed, because such methods require dense samples which would have an unreasonably large cardinality. The first result in this field, due to Dasgupta, used random projection onto a low dimensional subspace \cite{Dasgupta}. It was shown that a mixture of Gaussians with unit covariance in dimension $n$ could be provably well clustered if the separation between the means of the components was $O(\sqrt{n})$. The result was later improved by Dasgupta and Schulman \cite{DasGuptaSchulman} using a variant of EM for unit covariance Gaussians, and by Arora and Kannan \cite{Arora}, using a distance-based algorithm, for Gaussians with at most unit covariance. These methods, to correctly classify the components, require a $O(n^{1/4})$ separation between the centers of the Gaussians. For mixtures of unit covariance Gaussians, Vempala and Wang \cite{Vempala} used PCA to obtain a dimension-free separation bound that depends only on the number of components. Their method is based on the fact that the space spanned by the $k$ top singular vectors of the mixture's covariance matrix contains the centers of the components. Projecting to this space has the effect of reducing the variance of each component while maintaining the separation between the centers. Kannan et al. \cite{KannanVempala} extended this idea to mixtures of log concave distributions with at most unit covariance, also requiring a separation between the centers that depends only on the number of components. Achlioptas and McSherry \cite{Achlioptas} further improved the dependency of the separation bound on the number of components. A combination of PCA with a reweighting technique was proposed by Brubaker and Vempala \cite{BrubakerVempala}. This method is affine invariant and can deal with highly anisotropic inputs as a result. When applied to a sample from a mixture of two Gaussians, the algorithm correctly classifies the sample under the condition that there exists a half space containing most of the mass of one Gaussian and almost none of the other. 
Finally, a different family of approaches uses the moments of the mixture to learn the parameters of the components. Strong results have been obtained in this direction (see {\it e.g.} \cite{rge,BelkinS}). These methods do not require any separation assumption, however their downside is that they require a priori knowledge of a small parametric family containing the component's distributions. They also become inefficient when applied to high dimensional data, since the number of moments involved grows rapidly with the dimension. For example, the currently fastest algorithm \cite{rge} for learning mixtures of Gaussians runs in time $O(n^6)$.

\vspace{0.5cm}

Another possible approach to the analysis of mixtures uses kernel matrices. On a dataset  $\{x_1,...x_N\}$ of $N$ points in $\mathbb{R}^n$ a kernel function $k:\mathbb{R}^n\times \mathbb{R}^n\to\mathbb{R}$ defines a $N \times N$ kernel matrix 
whose $ij$ entry is $k(x_i,x_j)$. An important class of kernels are positive definite kernels, which are those for which the associated kernel matrix is positive definite for any dataset. 
The use of such kernel matrices, and in particular of their spectral decomposition as in the popular kernel PCA algorithm, has long become commonplace in data analysis. Still, surprisingly little is known regarding theoretical justifications for kernel based clustering methods. Notably, the analysis in \cite{Shietal} implies that a kernel PCA type algorithm will correctly cluster mixtures when the components are sufficiently separated. However, the arguments used in this paper follow the low (or constant) dimensional intuition and the required separation between the components is of the order of the width of the kernel, which typically leads to a separation that grows like the square root of the dimension. 

In order to improve the above analysis of kernel PCA, it is necessary to better understand the behavior of kernel matrices and of their spectra as the dimension increases. Again, while the literature on eigenvalues of random matrices is vast and growing rapidly, the knowledge about random kernel matrices is much scarcer. A notable exception is \cite{SKM}, which gives an asymptotic description of  radial kernel matrices of the form $k(x_i,x_j) = h(\|x_i-x_j\|^2/n)$ as the dimension $n$ tends to infinity, for a fixed function $h$. In the case of distributions whose coordinates are independent after some linear change of coordinate, {\it e.g.} Gaussians, it is shown that the kernel matrices converge in the operator norm to a certain matrix derived from the covariance of the data, suggesting that such kernels do not provide additional information compared to standard PCA. Under the weaker condition that the distribution enjoys concentration properties, the corresponding convergence result is proved to hold at the level of spectral distributions, but no result is derived for individual eigenvalues.

In this paper, we prove new results about radial kernel matrices of mixtures of high dimensional distributions. Unlike \cite{SKM}, we do not assume independence of coordinates. Rather, we only assume that the components in the mixtures have exponential concentration. Specifically, we show that such matrices can be very well approximated by the sum of a matrix that is row constant within each component and a matrix that is column constant within each component. For distance matrices of mixtures with a single component, the result implies a large spectral gap between the two largest  singular values: The ratio between these singular values is of the order of the dimension, rather than that of the square root of the dimension, as one might naively expect from basic concentration results. When the input distributions are supported on a sphere, this ``double concentration'' phenomenon is enhanced and large eigenvalue gaps arise for kernel matrices more general than distance matrices. The proof technique is geometric and very different from the one used in \cite{SKM}.

For positive kernels, a consequence of the above result is that kernel PCA is a valid clustering method as long as the Gram matrix of the mixture's components, when viewed as elements of the corresponding Reproducing Kernel Hilbert Space (RKHS), is sufficiently well conditioned. In particular, this allows to check that kernel PCA allows to correctly clusters mixtures of two Gaussians with a required separation between centers that does not depend on the dimension. 

In the case of even distributions supported on a sphere and satisfying
a Poincar\'e inequality, we further show that our main result can be
strengthened. In particular, kernel matrices of
the form $k(x_i,x_j) = h(\|x_i-x_j\|^2/\sqrt{n})$ are well approximated by block
constant matrices, provided $h$ is smooth enough. We also design a specific (non positive)
kernel of this form for which this result can be extended to non necessarily even
(and non necessarily centered) distributions. This
kernel, unlike the ones of the standard form studied in \cite{SKM}, can reveal
information beyond standard PCA in the asymptotic
regime. Specifically, we derive a
simple spectral algorithm for clustering mixtures with possibly common
means. This algorithm will succeed if the angle between any two
covariance matrices in the mixture (seen as vectors in
$\mathbb{R}^{n^2}$) is larger than $O(n^{-1/6}\log^{5/3}n)$. In
particular, the required angular separation tends to $0$ as the
dimension tends to infinity. To the best of our knowledge, this is the
first polynomial time algorithm for clustering such mixtures beyond
the Gaussian case.

Our results are described in the next three sections. The remaining
sections are devoted to their proofs. Throughout the paper, we write $f=O(g)$ to mean that there exists an absolute constant $c$
such that $f\leq cg$, and similarly for $\Omega()$ and $\Theta()$. A
statement holds with arbitrarily high probability if the probability
that it holds can be made arbitrarily high if the absolute constants
implicit in the involved $O()$, $\Omega()$ and $\Theta()$ notations are
chosen appropriately. 

\section{Kernels in high dimensions} 

Our analysis of kernel matrices for high dimensional data hinges on the concentration of measure phenomenon. Concentration of measure is a property of metric measure spaces that roughly says that regular functions are nearly constant \cite{Milman,Ledoux,Gromov}. It can be observed in many spaces, typical examples being Gaussian spaces or manifolds with Ricci curvature bounded below. We give precise definitions below for a probability measure $\mu$ on $\mathbb{R}^n$. We say that $f:\mathbb{R}^n\to \mathbb{R}$ has exponential $\sigma$-concentration, or $\sigma$-concentration for short, for some $\sigma>0$, if for any $\varepsilon > 0$: 
\begin{align*}
\mu \{ x: |f(x) - M(f)| \geq \varepsilon \} \leq O(e^{-\frac{{\varepsilon}}{\sigma}})
\end{align*}
where $M(f)$ is a median of $f$. 
The measure $\mu$ is said to have $\sigma$-concentration if all $1$-Lipschitz functions have $\sigma$-concentration. 
In particular, we have that $f$ equals $M(f)$ plus or minus $O(\sigma)$ with high probability. 

Levy's lemma \cite{cm} states that an isotropic Gaussian with covariance $\sigma^2I$ has Gaussian concentration, which is a stronger property implying $O(\sigma)$-concentration. This result is also true for anisotropic Gaussians if one takes $\sigma^2$ to be the maximum eigenvalue of the covariance matrix.
In particular, it implies that for high dimensional Gaussian spaces, most of the points are at about the same distance from the center. More precisely, almost all the mass of an isotropic Gaussian is concentrated in a spherical shell of radius $\sigma \sqrt{n}$ and thickness $O(\sigma)$. Indeed, for an isotropic Gaussian vector $x$, $\mathbb{E}(\|x\|^2)=\sigma^2n$. As distance functions are $1$-Lipschitz, by Levy's lemma, they have $O(\sigma)$-concentration. Hence the distance from a random point to the center differ by at most $O(\sigma)$ from $\sigma \sqrt{n}$, with high probability.

A stronger form of concentration that we will also consider is based on Poincar\'e inequality. We will say that a probability measure $\mu$ satisfies a Poincar\'e inequality if for any Lipschitz function $f:\mathbb{R}^n\to\mathbb{R}$ whose mean is zero with respect to $\mu$, we have
$$
\int f^2 d\mu \leq O(1) \int \|\nabla f\|^2 d\mu
$$ 
A probability measure that satisfies a Poincar\'e inequality necessarily has $O(1)$-concentration \cite{bobkov}. Gaussians distributions whose covariance have $O(1)$ eigenvalues are known to satisfy a Poincar\'e inequality. The famous KLS conjecture \cite{KLS} states that uniform distributions over isotropic convex bodies, and more generally isotropic measures with log-concave densities also do. 
 
\subsection{Main result}

We consider a mixture $\mu$ of  $k$ distributions $\mu_i$ in $\mathbb{R}^n$, with weights $w_i$, which we treat as numerical constants. We assume that each component $\mu_i$ has $O(1)$-concentration. Drawing a sample of $N$ points independently from the mixture gives a point set $X$ that is, with probability $1$, the disjoint union of subsets $X_i$, corresponding to each component. 
The \emph{radius} of $\mu_i$ is the quantity $(\mathbb{E}_{\mu_i} ||x-\mathbb{E}_{\mu_i} x||^2)^{1/2}$ for a random variable $x$ with law $\mu_i$, and we denote by $R$ the smallest radius of the $\mu_i$.  We consider a function $h : \mathbb{R}_{+} \rightarrow \mathbb{R}$  and the associated radial kernel. This defines a kernel matrix $\Phi_{h}(X)$ whose entries are $h(\|x_i-x_j\|)/N$, for $x_i, x_j$ in $X$. We assume that the indices are ordered in such a way that the components form contiguous intervals~; in particular, we have a natural block structure (doubly)-indexed by 
the components. 

\begin{theorem}\label{thm:thm1} If the number of samples $N$ is drawn according to the Poisson distribution with mean $N_0$, then with arbitrarily high probability, we have:
\begin{align*}
\|\Phi_h(X) - A\| \leq O\left(c_h+  \|h\|_\infty\sqrt{\frac{n\log N_0}{N_0}}\right)
\end{align*}
where $\|.\|$ denotes the operator norm, and the entries of $A$ in the $ij$ block are given by 
$$
A_{xy} = \frac{1}{N} \left(\int h(\|x-z\|)d\mu_j(z) + \int h(\|y-z'\|)d\mu_i(z') - \int h(\|z-z'\|)d\mu_i(z)d\mu_j(z')\right)
$$ 
and with
$$
c_h = \sup_{r \geq R/2} \left(|h''(r)| + \frac{1}{r} |h'(r)|\right) +  \|h'\|_\infty \exp\left(-\Theta(R)\right)\\
$$
Furthermore, if the components $\mu_i$ are supported on the sphere centered at $0$ with radius $\sqrt{n}$, and have mean at distance $O(1)$ from the origin, the conclusions above hold with $c_h$ replaced by
$$
c_h' = \sup_{r \geq R/\Theta(\log(R))} \left(\frac{ \log^2(R)|h''(r)|+|h'(r)|}{r} \right) + \|h'\|_\infty/{R}
$$
\end{theorem}
The proof of Theorem \ref{thm:thm1} follows from the analysis of the map sending each point $x$ in $\mathbb{R}^n$ to its kernel function $h(\|x-.\|)$ in $L^2(\mathbb{R}^n,\mu)$ or, more precisely, of a finite sample version of this map. That analysis crucially depends on the fact that in Euclidean spaces, the cross derivative of the distance $\frac{\partial^2}{\partial x \partial y}\|x-y\|$ is upper bounded by $O(1/\|x-y\|)$.  
A first consequence of Theorem $\ref{thm:thm1}$ is the following result about the spectrum of $\Phi_h(X)$, which follows directly from the variational characterization of eigenvalues:
\begin{corollary}\label{cor2} Under the assumptions of Theorem \ref{thm:thm1}, the spectrum of $\Phi_h(X)$ has at most $k$ eigenvalues larger than $O\left(c_h+  \|h\|_\infty\sqrt{\frac{n\log N_0}{N_0}}\right)$, and at most $k$ eigenvalues smaller than $-O\left(c_h+  \|h\|_\infty\sqrt{\frac{n\log N_0}{N_0}}\right)$, with arbitrarily high probability. 
\end{corollary}

\subsection{Distance matrices}

To illustrate Theorem \ref{thm:thm1}, setting for example $h(r)=r$ gives a description of distance matrices. Consider the case of a sample drawn from a mixture of $k$ Gaussians with unit covariance. If $x_i$ and $x_j$ are drawn independently from two Gaussians in the mixture, $x_i-x_j$ is a Gaussian with covariance $2I$.  Concentration of measure then implies that the entries $||x_i-x_j||$ of each block concentrate around their mean value, i.e. they differ by at most $O(1)$ from the mean of the block with high probability:

\begin{align}
 \Phi_h(X) = \label{mvs}
  \left(\begin{array}{c|c|c} 
\Phi_{11} & \cdot \cdot &  \Phi_{1k} \\ \hline
\vdots &  \cdot \cdot & \vdots \\ \hline
\Phi_{k1}& \cdot \cdot & \Phi_{kk} \\
\end{array} \right ) =
\frac{1}{N}\left(\begin{array}{c|c|c}
m_{1} \pm O(1) & \cdot \cdot & m_{1k} \pm O(1) \\ \hline
\vdots & \cdot \cdot & \vdots \\ \hline
m_{k1} \pm O(1) & \cdot \cdot & m_{kk} \pm O(1) \\ 
\end{array} \right ) 
\end{align}

A finer description of $\Phi_h(X)$ is given by Theorem \ref{thm:thm1}. For an isotropic Gaussian, the radius $R$ is $\Theta(\sqrt{n})$, and from $|h'| = 1$, $ |h''| = 0$ we get  
$
c_h = \Theta(1/\sqrt{n})
$. 

The dependency on the average number of samples $N_0$ in Theorem \ref{thm:thm1} involves $\|h\|_\infty$, which is unbounded. However, assuming for example that the centers of the components are at distance $O(1)$, then the fraction of pairs of sample points whose distance is larger than an appropriate constant times $\sqrt{n}$ is exponentially small by concentration. Hence we can first modify $h$ by thresholding such that $\|h\|_\infty$ becomes $O(\sqrt{n})$, with an exponentially small change in $\Phi_h(X)$. Furthermore, by making the transition between the linear part and the constant part  smooth enough, we can ensure that the second derivatives of the modified kernel $g$ are $O(1/\sqrt{n})$, so that $c_g=O(1/\sqrt{n})$. Applying the theorem to $g$ implies that with a polynomial number of samples ($N_0 = \Omega(n^3\log n)$ suffices), with arbitrarily high probability, each block of $\Phi_h(X)$ has the following structure
\begin{align*}
 \Phi_{ij} = 
  \frac{1}{N}\left(\begin{array}{cccc} 
a_{1} & a_{2} & \cdot \cdot &  a_{N_i} \\
a_{1} & a_{2} & \cdot \cdot &  a_{N_i} \\
\vdots & \vdots & \cdot \cdot & \vdots \\
a_{1} & a_{2} & \cdot \cdot & a_{N_i} \\
\end{array} \right ) +
\frac{1}{N}\left(\begin{array}{cccc}
b_{1}  & b_{1} & \cdot \cdot & b_{1}\\
b_{2} & b_{2} & \cdot \cdot & b_{2}\\
\vdots & \vdots & \cdot \cdot & \vdots \\
b_{N_j} & b_{N_j} & \cdot \cdot & b_{N_j} \\
\end{array} \right ) + B
\end{align*}
with $\|B\| = O(1/\sqrt{n})$. Note that the error term $B$ is now much smaller than the one in ($\ref{mvs}$), which is a priori up to $O(1)$ in the operator norm.

Furthermore, for each block the vectors $(a_s)$ and $(b_t)$ are, up to a constant, averages of the columns of the distance matrix. As a result these vectors are $1$-Lipschitz and thus have $O(1)$-concentration. Also, we can assume they have the same mean, namely half the average distance $m_{ij}$ within the block, that is, at least $\Omega(\sqrt{n})$. So we can write $a_s = m_{ij}(1+\varepsilon_s)/2$ and $b_t = m_{ij}(1+\delta_t)/2$ with $\varepsilon_s$ and $\delta_t$ in $O(1/\sqrt{n})$ with high probability. This implies that each block is very well approximated by a rank one matrix. Indeed 
$$a_s+b_t = m_{ij}(2+\varepsilon_s+\delta_t)/2 = m_{ij}((1+\varepsilon_s/2)(1+\delta_t/2)+O(1/n))$$
In particular, the normalized distance matrix of points drawn according to a single Gaussian has only one singular value that is larger than $O(1/\sqrt{n})$, this top singular value being $\Theta(\sqrt{n})$. This observation, which we stated for isotropic Gaussians for concreteness, applies to any distribution with $O(1)$-concentration and variance $\Theta(n)$ as well. 

We also remark that in the case of distributions on the sphere with $O(1)$-concentration and variance $\Theta(n)$, the contribution of $h''$ in the error bound in Theorem $\ref{thm:thm1}$ is divided by $\Omega(\sqrt{n}/\log^3 n)$, which makes it possible to extend the above discussion to kernels other than distance functions. We do not elaborate further as the spherical case will be studied in more detail in the sequel of the paper.

\section{Positive definite kernels and clustering}\label{secker}

For radial kernels that are positive definite, \emph{i.e.} that define positive definite kernel matrices, Corollary \ref{cor2} implies that there are at most $k$ significant eigenvalues for mixtures of $k$ probability measures that concentrate. We can use this result to provide guarantees for a simple clustering algorithm. First, assuming a certain gap condition, we can relate eigenspaces of the kernel matrix to the space of piecewise constant vectors, {\it i.e.} vectors that are constant on each component in the mixture. 

The required gap condition can be conveniently formulated in terms of \emph{kernel distances} \cite{suquet,bousquet}. Recall that kernel distances are Hilbertian metrics on the set of probability measures, which are obtained by embedding the ambiant Euclidean space into a universal RKHS. More precisely, given two probability measures $\mu_1$ and $\mu_2$ on $\mathbb{R}^n$, the expression
$$
\langle \mu_1,\mu_2 \rangle = \int h(||x-y||)\,d\mu_1(x)d\mu_2(y)
$$
is a positive definite kernel and the kernel distance is the associated distance. 

\begin{proposition}\label{prop:space} Assume $h$ defines a positive definite kernel, and that the conditions of Theorem \ref{thm:thm1} are satisfied. Let 
$$G_h = (\langle \mu_i,\mu_j\rangle)_{i,j=1\dots k}$$ be the Gram matrix of the components in the kernel distance. 

If the smallest eigenvalue of $G_h$ is at least $K c_h$, then the maximum angle formed by the space spanned by the top $k$ eigenvectors of $\Phi_h(X)$ and the space of piecewise constant vectors is at most  $O(1/ \sqrt{K})$, with arbitrarily high probability, provided $N_0\geq N_1$, with:
$$N_1=O\left(\frac{\|h'\|^2_\infty}{c_h^2}+\frac{n\|h\|^2_\infty}{c_h^2}\log\left(\frac{n\|h\|^2_\infty}{c_h^2}\right)\right)$$
\end{proposition}

Under these assumptions we can provide a guarantee for the following basic kernel PCA clustering algorithm. First, we perform a spectral embedding using the $k$ top eigenvectors of $\Phi_h(X)$. Namely, each data point $x$ is mapped to $(\phi_1(x),\dots,\phi_k(x))$, $\phi_1,\dots,\phi_k$ being the $k$ dominant eigenvectors of $\Phi_h(X)$. In order to have the right dependency on the total number of points, these eigenvectors are scaled to have norm $\sqrt{N}$. By the above proposition, this will give a point cloud that is $O(\sqrt{1/K})$ close in the transportation distance $W_2$
 to a point cloud obtained using the embedding provided by an orthogonal basis of piecewise constant vectors, scaled to have norm $\sqrt{N}$. Note that in the latter point cloud, each component becomes concentrated at a single location, the distance between any two such locations being $\Omega(1)$. In such a situation, any constant factor approximation algorithm for the $k$-means problem will find a clustering with a fraction of at most $O(1/K)$ misclassified points.
We just proved:

\begin{corollary}\label{cor}
If the assumptions of Proposition \ref{prop:space} are satisfied, kernel PCA allows to correctly cluster a $1-O(1/K)$ fraction of the mixture, with arbitrarily high probability.
\end{corollary}

As an example, we consider the case of a mixture of two Gaussians using a Gaussian kernel $h(r) = \exp(-r^2/(2\tau^2))$. In this case, matrix $G_h$ can be computed in closed form, so that the conditions of Proposition \ref{prop:space} can be checked explicitly.

\begin{corollary}\label{cor:g}
Consider a mixture of two Gaussians with $O(1)$-concentration in $\mathbb{R}^n$. Assuming that the variance of each Gaussian is $\Theta(n)$, for $\tau=\Theta(\sqrt{n})$, Gaussian kernel PCA allows to correctly cluster a $1-O(1/K)$ fraction of the mixture if the distance between the centers is $K$. 
\end{corollary}

The choice of variance for the components in the above corollary is to fix ideas, similar conclusions would hold with other behaviors. The above guarantee matches the dimension-independent separation required by the PCA-based algorithms described in \cite{KannanVempala,Achlioptas} for example. Finally, the results in this section are in fact not strongly tied to the Hilbertian nature of positive kernels. More precisely, they may be easily extended to conditionally positive kernels, by simply restricting the involved quadratic forms to the space of zero mean functions. We omit further details.

\section{Covariance based clustering}

As shown in the above section, the approximation of kernel matrices provided by Theorem \ref{thm:thm1} is sufficient to conclude that their top eigenvectors are nearly constant on the clusters if the kernel is positive, which allows to correctly cluster the data. Unfortunately,  while we showed that positive kernels could allow to cluster {\it e.g.} mixtures of Gaussians with different enough centers, the range of cases that can be successfully clustered using positive kernels remains unclear at this stage. In this section we show that by relaxing the positivity constraint, one can design kernels that can deal with more difficult situations, such as mixtures of distributions with common centers but different covariances. While Theorem \ref{thm:thm1} alone is insufficient for this purpose, we show that stronger conclusions can be obtained assuming that the components of the mixtures are supported on the sphere $S$ with radius $\sqrt{n}$ and centered at the origin, and satisfy a Poincar\'e inequality. Namely, kernel matrices can then be approximated by block constant matrices, rather than a sum of column and row constant matrices within each block. We state below such a result for general kernels, assuming the input distributions are even. We also consider the case of non necessarily even distributions with small enough means. Similar conclusions can then be drawn for the kernel
$$
h_t(r) = \cos\left(\frac{t}{\sqrt{n}}(n-r^2/2)\right)
$$
where $t$ is a parameter. The argument is more direct and avoids the use of Poincar\'e inequality. A more transparent way to write this kernel is to remark that for $x$ and $y$ on $S$, $$h_t(\|x-y\|)=\cos\left(\frac{t}{\sqrt{n}}<x,y>\right)$$
Note that $h_t$ has a perhaps non intuitive behavior compared to the most commonly used kernels as it oscillates $\Theta(\sqrt{n})$ times over the sphere $S$ for $t=\Theta(1)$ for example. 

\begin{theorem}\label{thm:thm2} Assume measures $\mu_i$ are supported on $S$, even, and satisfy a Poincar\'e inequality. Let $\tilde{h}(r)=h'(r)/r$. If the number of samples $N$ is drawn according to the Poisson distribution with mean $N_0$, then with arbitrarily high probability, we have:
\begin{align*}
\|\Phi_h(X) - B\| \leq O\left(c'_h+\sqrt{n}c'_{\tilde{h}}+  \|h\|_\infty\sqrt{\frac{n\log N_0}{N_0}}\right)
\end{align*}
where $\|.\|$ denotes the operator norm, and the entries of $B$ in the $ij$ block are all equal to
$$
G_h(i,j)/N = \frac{1}{N} \left(\int h(\|z-z'\|)d\mu_i(z)d\mu_j(z')\right)
$$ 
For the kernel $h_t$, under the weaker assumption that measures $\mu_i$ are supported on $S$, have $O(1)$-concentration and have means at distance $O(1)$ from the origin, we have:
\begin{align*}
\|\Phi_{h_t}(X) - B\| \leq  O\left(\frac{t\log^3n}{\sqrt{n}}+\sqrt{\frac{n\log N_0}{N_0}}\right)
\end{align*}
with arbitrarily high probability for $t=O(1)$. 
\end{theorem}

In particular, in the case of even distributions satisfying a
Poincar\'e inequality, letting the sample size go to infinity and 
expliciting the upper bound in the first part of the theorem gives:

\begin{corollary}\label{correigen}
  For a fixed bounded function $h$ with bounded derivatives up
  to the third order, the radial convolution operator from
  $L^2(\mathbb{R}^n,\mu_i)$ to $L^2(\mathbb{R}^n,\mu_j)$ associated
  with kernel $r \mapsto h(r ^2/\sqrt{n})$ has at most one singular
  value larger than $O(\log^3 n/\sqrt{n})$.
\end{corollary}

  It seems likely that the logarithmic factor can in fact be removed, by replacing the Lipschitz extension argument by a Dirichlet energy estimate in the proof of Theorem \ref{thm:thm1}. 
\newline

We now show that the second part of the above theorem can be used to cluster high dimensional mixtures based on the components covariance matrices. We assume that the components $\mu_i$ have $O(1)$-concentration and variance $\Theta(n)$. As the PCA algorithm of \cite{KannanVempala} allows to separate components whose means are at distance at least $\Omega(1)$ from the other means, it is sufficient to consider the case where all means are at distance $O(1)$ from the origin. We denote by $\Sigma_i$ the non centered covariance matrix of $\mu_i$. Given $s>0$ and a symmetric matrix $M$, we define $f_s(M)$ to be the matrix having the same eigenvectors as $M$, eigenvalues being transformed by function $\lambda\mapsto f_{s}(\lambda)$, with $f_{s}(\lambda) = \max(0,|\lambda|-s)$. Let $$\Delta = \sqrt{n}\min_{u\neq v}\left\|\frac{\Sigma_u}{\mathrm{trace}\Sigma_u}-\frac{\Sigma_v}{\mathrm{trace}\Sigma_v}\right\|_2$$
As covariance matrices have trace $\Theta(n)$, they have Frobenius norm $\Theta(\sqrt{n})$, so that $\Delta = \Omega(\alpha_{min})$, $\alpha_{min}$ being the minimum angle between any two covariance matrices. Let further  $C_1,C_2$ be two appropriate universal constants. The algorithm we propose is the following:

\begin{minipage}{13cm}
\begin{flushleft}
\begin{algorithm}[H]
 \caption{CovarianceClustering($X$)}
    \label{alg:algorithm-label}
\begin{algorithmic}
\STATE $\tilde{X} =$ data points projected on $S$
\STATE $\Phi = \Phi_{h_t}(\tilde{X})$, with $t = C_1 \Delta$
\STATE Approximately solve the k-means problem for the columns of $f_{C_2\Delta^4}(\Phi)$
\end{algorithmic}
\end{algorithm}
\end{flushleft}
\end{minipage}

\mbox{ }\newline

To prove that this algorithm succeeds, we apply Theorem \ref{thm:thm2} to the data projected on $S$, which tells us that $\Phi_{h_t}(\tilde{X})$ is well approximated by block constant matrix $B$. We then show that under our separation assumptions, matrices $G_{h_t}$ are well-conditioned in the case of mixtures of two components. Using this fact, we show that the columns $f_{C_2\Delta^4}(B)$ corresponding to different components are sufficiently far apart. Applying a perturbation bound then allows to conclude, and obtain the following guarantee:

\begin{theorem}\label{prop:space2} If $\Delta\geq Kn^{-1/6}\log^{5/3}n$, the above algorithm allows to correctly cluster a $O(1/K^6)$ fraction of the mixture with arbitrarily high probability, provided $N_0\geq N_1$, with:
$$N_1 =O \left(\log(n/\Delta)n^2/\Delta^2\right)$$
\end{theorem}

Hence clustering will succeed if the minimum angle $\alpha_{min}$ between the components covariances is larger than $O(n^{-1/6}\log^{5/3}n)$. First note that one case is not covered by this algorithm, namely the case where different components have covariance matrices differing only by a scaling. This situation can be dealt with easily by clustering the data according to the distance to the origin. A second remark can be made about the sample size. The guarantee given above aims for the smallest angular separation, and as a result requires a number of points that is more than quadratic in the dimension. While it is possible that a better analysis would give smaller sample sizes in this regime, we remark that if $\alpha_{min}=\Omega(1)$, the proof can be modified to show that correct clustering will require only $O(n\log n)$ points. Indeed, in this situation, the error bound in Theorem \ref{thm:thm2} is dominated by the contribution of the sample size, and having $O(n\log n)$ points will make it small enough so that the rest of the analysis can be applied.

To conclude, we give some numerical results on specific examples of equal weight mixtures of two Gaussian distributions $\mu_1$ and $\mu_2$ with mean zero on $\mathbb{R}^{n}$, with $n$ even. The covariances $\Sigma_1$ and $\Sigma_2$ are both diagonal in the standard basis. For a parameter $s>0$, the eigenvalues of $\Sigma_1$ are $1+s$ on the first $n/2$ coordinates, and $1-s$ on the last $n/2$ coordinates. Eigenvalues of $\Sigma_2$ are reversed, so that $\Sigma_1+\Sigma_2=2I$, meaning that the whole distribution is isotropic. Under the assumptions of Theorem \ref{prop:space2}, as shown in the proof, the spectral soft thresholding operation used in the algorithm will leave at most $2$ non zero eigenvalues. Rather than implementing the full algorithm, we just plot the second dominant  singular vector of $\Phi$, as the first one turns out not to separate the components. Figure \ref{fig:1} shows it for $s=0.9,\, n=10$, $s=0.6,\, n=100$, $s=0.33,\, n=1000$ and $s=0.2,\,n=10000$, with $t=0.1$. In all cases each Gaussian has $n$ sample points. We see that the clusters are easily detected. Note that in the latter case, the Gaussians are nearly spherical, the relative error being of roughly $10\%$ in terms of standard deviation. 

\begin{figure}
\centering
  \includegraphics[width=.35\linewidth]{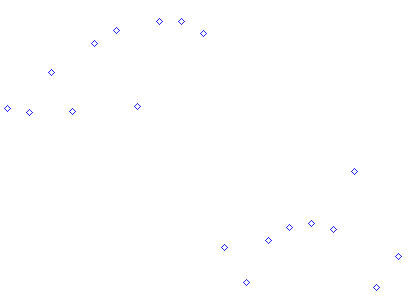}
  \centering
\hspace{0.5cm}
  \includegraphics[width=.35\linewidth]{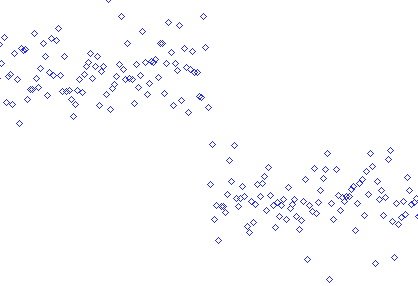}
\hspace{0.5cm}
\centering
  \includegraphics[width=.35\linewidth]{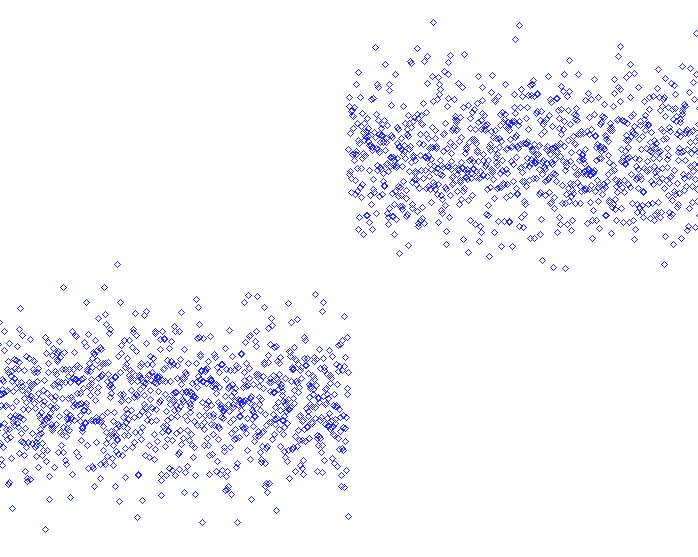}
  \centering
\hspace{0.5cm}
  \includegraphics[width=.35\linewidth]{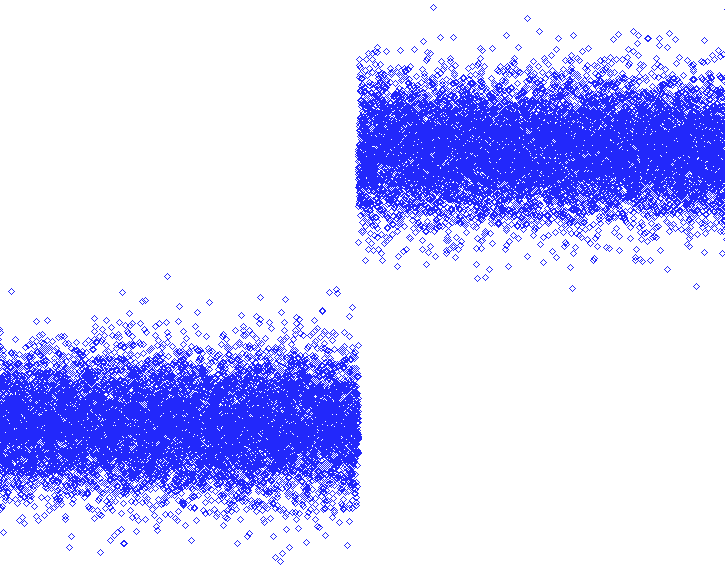}
\caption{Second singular vector of $\Phi$ for isotropic mixtures of centered Gaussians.}
\label{fig:1}
\end{figure}

 
\section{Proof of Theorem $\ref{thm:thm1}$ }

For technical reasons we will not work directly with the input measure $\mu$, but rather with its empirical measure $\bar{\mu} = \sum_i w_i\bar{\mu}_i$, the number of samples being drawn according to a Poisson distribution with appropriately large mean $M_0$. Since the $\mu_i$ have $O(1)$-concentration, a vector $X$ with law $\mu_i$ satisfies $\mathbb{E}(|X-\mathbb{E}X\|^q)^{1/q} = O(\sqrt{n})$ for constant $q\geq 1$,  which implies (see {\it e.g.} \cite{Fournier}) that 
$$
\mathbb{E}(W_l(\mu_i,\bar{\mu}_i)) = O(nM_0^{-1/n})
$$
where $W_l$ are the transportation distances for $l=1$ or $2$. By Markov inequality, for any $\delta>0$, these distances are at most $\delta$ with probability at least $1-p$, with
$$
p = O\left(\frac{nM_0^{-1/n}}{\delta}\right)
$$

Consider the map
\begin{align*} 
  \phi_{\bar{\mu}} :\,  & \mathbb{R}^n \rightarrow L^2(\mathbb{R}^n,\bar{\mu}) \\
                  &  x \mapsto \phi_{\bar{\mu}}(x) = h(\|x- \cdot\|)   
\end{align*}    

The gist of our proof of Theorem \ref{thm:thm1} is as follows. We first observe that the directional derivatives of $\phi_{\bar{\mu}}$ at each point satisfy a Lipschitz condition with a small constant. More precisely, this is true after modifying them in a small region, which is enough for our purposes. Using concentration of measure, this implies that these derivatives, modulo piecewise constant functions on the components, are small. This can be further reinterpretated as saying that $\phi_{\bar{\mu}}$, after centering on each component, has a small Lipschitz constant. Because each component has constant concentration by assumption, this implies that the image of each component by $\phi_{\bar{\mu}}$, after centering on each component, has small concentration. The desired claim on the block structure of $\Phi_h(X)$ can then be deduced. 

\subsection{A property of $\phi_{\bar{\mu}}$} Let $E \in L^2(\mathbb{R}^n,\bar{\mu})$ be the space of functions that are constant on the support of each $\bar{\mu}_i$, and $P_{E}$ and $P_{E^{\perp}}$ denote the orthogonal projectors onto $E$ and $E^{\perp}$. Further denote by $S$ the sphere with radius $\sqrt{n}$ centered at $0$.

\begin{proposition}\label{prop:prop3} With probability at least $1-p$, for any  $x_1$ and $x_2 \in \mathbb{R}^n$, 
$$\| P_{E^{\perp}}\phi_{\bar{\mu}}(x_1) - P_{E^{\perp}}\phi_{\bar{\mu}}(x_2)\| \leq O(c_h(\delta)) \|x_1-x_2\|$$
Furthermore if measures $\mu_i$ are supported on $S$ and their mean is $O(1)$, then with probability at least $1-p$, for any $x_1$ and $x_2$ in $S$:
$$\| P_{E^{\perp}}\phi_{\bar{\mu}}(x_1) - P_{E^{\perp}}\phi_{\bar{\mu}}(x_2)\| \leq O(c_h'(\delta)) \|x_1-x_2\|$$
with 
\begin{eqnarray*}
c_h(\delta) &=& (1+\delta)c_h + \sqrt{\delta}\|h'\|_\infty\\
c_h'(\delta) &=& (1+\delta)c_h' + \sqrt{\delta}\|h'\|_\infty
\end{eqnarray*} 
\end{proposition} 

\noindent To prove the first part of Proposition \ref{prop:prop3} we argue that

\begin{align*} 
\| P_{E^{\perp}} \phi_{\bar{\mu}}(x_1) - P_{E^{\perp}} \phi_{\bar{\mu}}(x_2) \|_2  & \leq  \sup_{x_0,v,\|v\|=1}\left \|\left.\frac{d}{dx} \right|_{v,x=x_0} P_{E^{\perp}}\phi_{\bar{\mu}} \right\|_2 \, \|x_1 - x_2\| \\
                                                                                        & \leq  \sup_{x_0,v,\|v\|=1}\left \|P_{E^{\perp}} \left.\frac{d}{dx} \right|_{v,x=x_0} \phi_{\bar{\mu}} \right\|_2 \,\|x_1 - x_2\| \\
                                                                                        & \leq \sup_{x_0,v,\|v\|=1} \left(\sum_i w_i \left \| f_i - \int f_i d\bar{\mu}_i\right \|^2_2\right)^{1/2} \,\|x_1 -x_2\|                                                              
\end{align*} 
where in the last line $f_i$ denotes the directional derivative of $\phi_{\bar{\mu}_i}$ at $x_0$ in direction $v$. To conclude, it is sufficient to prove that 
\begin{align}
\sup_{x_0,v,\|v\|=1} \left \| f_i - \int f_i d\bar{\mu}_i \right \|_2\leq O(c_h(\delta)) \label{phi}
\end{align}
For the second part, we use a similar argument except that we interpolate between $x_1$ and $x_2$ using a great circle on $S$ instead of a straight line. This shows that establishing
\begin{align}
\sup_{x_0,v,\|v\|=1,\langle v,x_0\rangle=0} \left \| f_i - \int f_i d\bar{\mu}_i \right \|_2\leq O(c_h'(\delta)) \label{phi'}
\end{align}
suffices to conclude. Proving these two inequalities is the point of the rest of this section. For some $\rho>0$, let 
$$L_{v,\rho} = \{ y |\; |\langle y, v\rangle|\leq 1/\rho\}$$
Further define:
\begin{eqnarray*}
d_h &=& \sup_{r \geq \rho R} \left(|h''(r)| + \frac{1}{r} |h'(r)|\right)\\
d_h' &=& \sup_{r \geq \rho R} \left(\frac{|h''(r)|}{\rho^2 r} + \frac{1}{r} |h'(r)|\right)
\end{eqnarray*}

\begin{lemma}
Function $f_i$ is $d_h$-Lipschitz outside $B(x_0,\rho R)$ and $|f_i|$ is bounded everywhere by $\sup |h'|$. Furthermore, if $v$ is a unit tangent vector at $x_0\in S$, then $f_i$ is $d'_h$-Lipschitz on $L_{v,\rho}\setminus B(x_0,\rho R)$.
\end{lemma}
\begin{proof}
We consider the radial coordinate system ($r,\theta$) centered at $x_0$, where $\theta$ denotes the angle formed by $y-x_0$ and $v$. 
A direct calculation shows that 
\begin{align*} 
f_i(y) =   \left. \frac{\mathrm{d}}{\mathrm{d}x}\right|_{v,x=x_0} \phi_{\bar{\mu}_i}(x)(y)  = h'(r) \cos \theta 
\end{align*}
Hence
\begin{align*} 
& \frac{d}{dr} f_i(r(y),\theta(y)) = h''(r) \cos \theta \\
& \frac{d}{d \theta} f_i(r(y),\theta(y))  = -h'(r) \sin \theta
\end{align*}
Noticing that $r$ is a $1$-Lipschitz function of $y$, and that $|d\theta/dy|\leq 1/r$ allows to bound the derivatives of $f_i$ in the radial and tangent directions using the chain rule, implying:
\begin{eqnarray*}
\|\nabla f_i(y)\| &=& \left( (h''(r(y)) \cos \theta(y))^2 +\left( \frac{h'(r(y))}{r(y)} \sin \theta(y)\right)^2\right)^{1/2}\\
                        &\leq& \max \left(h''(r(y))\, |\cos \theta(y)|,\frac{h'(r(y))}{r(y)}\right)
\end{eqnarray*}
Using that $|\cos \theta(y)|\leq 1/(\rho^2 r)$ on $L_{v,\rho}\setminus B(x_0,\rho R)$, the conclusion follows.
\qed\end{proof}

\begin{lemma}\label{lem:decfi} We can write $f_i = \tilde{f}_i + g_i$, where $\tilde{f_i}$ is $d_h$-Lipschitz, and $g_i$ is supported on $B(x_0,\rho R)$ with $||g_i||_{\infty} \leq 2\sup_{r} |h'(r)|$. If $v$ is a unit tangent vector at $x_0\in S$, then we can find a similar decomposition with $\tilde{f_i}$ $d'_h$-Lipschitz and $g_i$ supported on $B(x_0,\rho R) \cup \mathbb{R}^n\setminus L_{v,\rho}$ with $||g_i||_{\infty} \leq 2\sup_{r} |h'(r)|$.
\end{lemma}

\begin{proof}
Define $\tilde{f_i}$ to be a $d_h$-Lipschitz extension of $f_i|_{\mathbb{R}^n \backslash B(x_0,\rho R)}$ to $\mathbb{R}^n$, which exists by Kirszbraun's extension theorem \cite{K}. We choose $\tilde{f}_i$ such that $\sup_{B(x_0,\rho R)} |\tilde{f_i}| = \sup_{\partial B(x_0,\rho R)} |\tilde{f_i}|$, which can be done by thresholding if necessary. The result follows by letting $g_i = f_i-\tilde{f_i}$. The spherical case is proved similarly.
\qed\end{proof}

\begin{lemma} \label{12}
With probability at least $1-p$, we have $\mathrm{Var}_{\bar{\mu}_i}(f_i)=O(c_h(\delta)^2)$. If measures $\mu_i$ are supported on $S$ with mean $O(1)$, then with probability at least $1-p$, for $v$ a unit tangent vector at $x_0\in S$, $\mathrm{Var}_{\bar{\mu}_i}(f_i)=O( c_h'(\delta)^2)$.
\end{lemma}

\begin{proof}
For the first claim, we write
\begin{eqnarray*}
\sqrt{\mathrm{Var}_{\bar{\mu}_i}(f_i)} &\leq& \sqrt{\mathrm{Var}_{\bar{\mu}_i}(\tilde{f_i})} + \sqrt{\mathrm{Var}_{\bar{\mu}_i}(g_i)}\\
                  &\leq& \sqrt{\mathrm{Var}_{\bar{\mu}_i}(\tilde{f_i})} + ||g_i||_2\\
                  &\leq& \sqrt{\mathrm{Var}_{\bar{\mu}_i}(\tilde{f_i})} + \sup|g_i| \bar{\mu}_i(B(x_0,\rho R))^{1/2}
\end{eqnarray*}
Because $\tilde{f_i}$ is $d_h$-Lipschitz, the pushforwards of $\mu_i$ and $\bar{\mu}_i$ satisfy $$W_2(\tilde{f}_{i\sharp} \bar{\mu}_i, \tilde{f}_{i\sharp} \mu_i) \leq d_h W_2(\bar{\mu}_i, \mu_i)\leq d_h\delta$$ And since $\mu_i$ has $O(1)$-concentration,  $\tilde{f}_{i\sharp} \mu_i$ has at most $O(d_h^2)$ variance. As a result $$\mathrm{Var}_{\bar{\mu}_i}(\tilde{f_i}) = \mathrm{Var} \tilde{f}_{i\sharp} \bar{\mu}_i \leq O((1+\delta^2) d_h^2)$$ Also, letting $d_{x_0}$ be the distance function to $x_0$, we have that $$W_1(d_{x_0\sharp}(\bar{\mu}_i),d_{x_0\sharp}(\mu_i)) \leq \delta$$ since distance functions are $1$-Lipschitz. Consider an optimal coupling $(X,Y)$ between $d_{x_0\sharp}(\bar{\mu}_i)$ and $d_{x_0\sharp}(\mu_i)$. By Markov inequality, the probability that $X\leq\rho R$ and $Y\geq \rho R+1$ is at most $\delta$. This implies that
$$
\bar{\mu}_i(B(x_o,\rho R)) \leq \delta + \mu_i(B(x_o,\rho R+1))
$$
Since $d_{x_0}$ $O(1)$-concentrates on $\mu_i$, its median is $O(1)$ close to $(\int d_{x_0}^2d\mu_i)^{1/2}$. As the latter quantity is at least $R$, we have by concentration
$$
\mu_i(B(x_0,\rho R+1)) \leq \exp\left(-\Omega(1-\rho) R+O(1)\right)
$$
As a consequence
$$
\sqrt{\mathrm{Var}_{\bar{\mu}_i}(f_i)} \leq  O\left((1+\delta^2)^{1/2} d_h +\sup_{r} |h'(r)| \left(\delta +\exp\left(-\Omega(1-\rho)R\right)\right)^{1/2}\right)
$$
The first claim follows by setting $\rho=1/2$. The spherical case is proved similarly, except that we use the inequalities
$$\bar{\mu}_i \left(B(x_0,\rho R) \cup \mathbb{R}^n\setminus L_{v,\rho}\right) \leq \bar{\mu}_i \left(B(x_0,\rho R)\right) + \bar{\mu}_i\left(\mathbb{R}^n\setminus L_{v,\rho}\right)$$
and
\begin{eqnarray*}
\bar{\mu}_i \left(\mathbb{R}^n\setminus L_{v,\rho}\right) &\leq& \delta + \mu_i\left(\{ y |\; |\langle y, v\rangle|\geq 1/\rho - 1\}\right)\\
&\leq& \delta + 2\exp \left(-\Omega(1/\rho) +O(1)\right)\\
&\leq& \delta + O(\exp (-\Omega(1/\rho)))
\end{eqnarray*}
which follows as above from the fact that linear functions $O(1)$-concentrate on $\mu_i$ and have mean $O(1)$.
Choosing appropriate $\rho = \Theta(\log(R)^{-1})$, the bound above becomes $\delta+1/R^2$, hence
\begin{eqnarray*}
\sqrt{\mathrm{Var}_{\bar{\mu}_i}(f_i)} &\leq&  O\left((1+\delta^2)^{1/2} d_h' +\sup_{r} |h'(r)| \left(\delta +1/R^2\right)^{1/2}\right)\\
&\leq& O(c_h'(\delta))
\end{eqnarray*}

\end{proof}
This proves (\ref{phi}) and (\ref{phi'}) and concludes the proof of Proposition \ref{prop:prop3}.

\subsection{Decomposition of $\Phi_h(X)$}

We first show the following variant of Theorem \ref{thm:thm1}:
\begin{proposition}\label{propp}
If the number of samples $M$ is drawn according to the Poisson distribution with mean $M_0$, then with probability at least $1-p$, we have $\|\Phi_h(X)-A\| = O(e_h(\delta))$ with $e_h(\delta) = c_h(\delta)(1+\delta)+\delta\|h'\|_\infty$, and $\|\Phi_h(X)-A\| = e'_h(\delta)$ with $e'_h(\delta) = c_h'(\delta)(1+\delta)+\delta\|h'\|_\infty$ in the spherical case.
\end{proposition}

The argument is the same for the spherical and for the non-spherical case, so we only consider the non spherical case. Let $M$ be the number of samples of $\bar{\mu}$. First decompose the unnormalized kernel matrix $D_h(X)=M\Phi_h(X)$ as follows:
\begin{align*}
D_h(X) = P_E D_h(X) + P_{E^{\perp}} D_h(X) 
\end{align*}
The first term $P_ED_h(X)$ is column constant within each block. We now focus on the second one.

\begin{lemma}\label{varcol}
With probability at least $1-p$, the centered covariance matrix of the columns of $P_{E^{\perp}} D_h(X)$ corresponding to any component has eigenvalues at most $O(Mc_h(\delta)^2(1+\delta^2))$.
\end{lemma}
\begin{proof}
The columns of $P_{E^{\perp}} D_h(X)$ are the images of the sample points by $P_{E^{\perp}} \phi_{\bar{\mu}}$, expressed in the standard basis. 
Hence by Proposition \ref{prop:prop3}, the map $\bar{\phi}$ associating each sample point with its column in $P_{E^{\perp}} D_h(X)$ is $O(\sqrt{M}c_h(\delta))$-Lipschitz with probability at least $1-p$. Let $\tilde{\phi}$ be a $O(\sqrt{M}c_h(\delta))$-Lipschitz extension of $\bar{\phi}$ to $\mathbb{R}^n$. Consider a unit vector $v\in \mathbb{R}^M$ and let $U$ be a random column of $P_{E^{\perp}} D_h(X)$. Variable $\langle U,v\rangle$ is equal to $\langle \bar{\phi}(V),v\rangle=\langle \tilde{\phi}(V),v\rangle$, where $V$ is drawn according to $\bar{\mu}_i$. Let now $W$ be drawn according to $\mu_i$. Since $\mu_i$ has $O(1)$-concentration, $\langle \tilde{\phi}(W),v\rangle$ has variance $O(Mc_h(\delta)^2)$. Because with probability at least $1-p$, $W_2(\bar{\mu}_i,\mu_i)< \delta$, the distributions of $\langle \tilde{\phi}(W),v\rangle$ and $\langle \tilde{\phi}(V),v\rangle$ are $O(\sqrt{M}c_h(\delta)\delta)$ away in the $W_2$ distance. As a consequence
$$
\mathrm{Var}(\langle \tilde{\phi}(V),v\rangle) = O(\mathrm{Var}(\langle \tilde{\phi}(V),v\rangle) + Mc_h(\delta)^2\delta^2) = O(Mc_h(\delta)^2(1+\delta^2))
$$
\end{proof}

Let us further decompose  
$$P_{E^{\perp}} D_h(X) =  P_{E^{\perp}} D_h(X) P_{E} + P_{E^{\perp}} D_h(X) P_{E^{\perp}} $$

as a sum of matrix $P_{E^{\perp}} D_h(X) P_{E}$ which is row constant within each block, and a remainder $M.B = P_{E^{\perp}} D_h(X) P_{E^{\perp}}$ whose columns are the columns of $P_{E^{\perp}} D_h(X)$ centered in each block. By Lemma \ref{varcol}, the non centered covariance matrix of all the columns of $M.B$ has eigenvalues at most $O(Mc_h(\delta)^2(1+\delta^2))$. As this covariance matrix is $M.BB^t$, this shows that $||B|| = O(c_h(\delta)(1+\delta))$. Thus we get:
$$
\Phi_h(X) = P_E\Phi_h(X) +  P_{E^{\perp}} \Phi_h(X) P_{E} + B
$$
Letting $\bar{A} = P_E\Phi_h(X) +  P_{E^{\perp}} \Phi_h(X) P_{E}$, we see that for $x\in \mathrm{support}(\bar{\mu}_i)$ and $y\in \mathrm{support}(\bar{\mu}_j)$, the $xy$ entry of $\bar{A}$ is given by
$$
M.\bar{A}_{xy} = \int h(\|x-z\|)d\bar{\mu}_j(z) + \int h(\|y-z'\|)d\bar{\mu}_i(z') - \int h(\|z-z'\|)d\bar{\mu}_i(z)d\bar{\mu}_j(z')
$$ 
By Kantorovich-Rubinstein theorem,
 $$\|A-\bar{A}\|\leq\sup_{xy}|M.\bar{A}_{xy}-M.A_{xy}|\leq O(\delta\|h'\|_\infty)$$ 
which concludes the proof.

\subsection{Sample size}
In order to prove that Theorem \ref{thm:thm1} also holds for small sample size, we use the following result in \cite{Tropp}. For a random variable $W$, let $E_k W$ denotes the $L_k$ norm of $W$. For a matrix $U$, $\|U\|_\infty$ is the maximum entry of $U$, and $\|U\|_{1,2}$ is the maximum norm of the columns of $U$.
\begin{theorem*}
Let $Z$ be a $M\times M$ Hermitian matrix, decomposed into diagonal and off-diagonal parts: $Z = D+H$. Fix $k$ in $[2,\infty)$, and set $q=\max\{k,2\log M\}$. Then
$$
E_k\|RZR\| \leq O\left(qE_k\|RHR\|_{\infty} + \sqrt{\eta q}E_k\|HR\|_{1,2} + \eta \|H\|\right) + E_k\|RDR\|
$$ 
where $R$ is a diagonal matrix with independent $0-1$ entries with mean $\eta$. 
\end{theorem*}
Let us apply this theorem to $Z=M(\Phi_h(X) - A_M)$, where $X$ is an iid sample of $\mu$ with cardinality $M$ distributed according to a Poisson distribution with mean $M_0$, and $A_M$ is the matrix specified in Theorem \ref{thm:thm1}. In any case $\|RZR\|\leq O(\mathrm{trace}(R)\|h\|_\infty)$, and by Proposition \ref{propp}, with probability at least $1-p$, we have $\|Z\|\leq O(Me_h(\delta))$ (and similarly for the spherical case). Clearly $E_k\|RDR\|$ and  $E_k\|RHR\|_{\infty}$ are both bounded by $O(\|h\|_\infty)$, and $E_k\|HR\|_{1,2}$ is at most $O(\|h\|_\infty E_k\sqrt{M})$. Also $\|H\| \leq \|Z\|+\|D\| \leq O(Me_h(\delta)+\|h\|_\infty)$ with probability at least $1-p$. Hence the theorem above gives:
$$
E_k\|RZR\| \leq \|h\|_\infty O(p E_k \mathrm{trace}(R) + q + \sqrt{\eta q}E_k\sqrt{M} + \eta) + O(\eta e_h(\delta)E_kM)
$$
Taking $k=2$ and $\eta = N_0/M_0$, we have $E_k \mathrm{trace}(R) = O(N_0)$, $E_k\sqrt{M} = O(\sqrt{M_0})$ and $E_kM = O(M_0)$. With $q=2\log M_0$, we get
\begin{eqnarray*}
E_2\left(\frac{\|RZR\|}{\mathrm{trace}(R)}\right) &\leq& O\left( E_2\left(\frac{\|RZR\|}{N_0} \right)\right)\\
&\leq& \|h\|_\infty O\left(p + \frac{\log M_0}{N_0} + \sqrt{\frac{\log M_0}{N_0}} + \frac{1}{M_0}\right) + O(e_h(\delta))\\
&\leq& \|h\|_\infty O\left(\frac{n}{\delta M_0^{1/n}} + \sqrt{\frac{\log M_0}{N_0}}\right) + O(e_h(\delta))\\
&\leq& \|h\|_\infty O\left(\frac{n}{\delta M_0^{1/n}} + \sqrt{\frac{\log M_0}{N_0}} + \delta + (1+\delta)\sqrt{\delta}\right) + (1+\delta)^2O(c_h)\\
\end{eqnarray*}
assuming $N_0\geq \log M_0$. Matrix $RZR/\mathrm{trace}(R)$ is simply $\Phi_h(Y)-A_N$, where $Y$ is an iid sample of $\mu$ with cardinality $N$  distributed according to a Poisson distribution with mean $N_0$. Continuing the last equation, taking $M_0=N_0^{3n/2}$ and $\delta = (n/M_0^{1/n})^{2/3}$ so that $p=\Theta(\sqrt\delta)$, we have
\begin{eqnarray*}
E_2(\|\Phi_h(Y)-A_N\|) &\leq& \|h\|_\infty O\left(\frac{n}{\delta M_0^{1/n}} + \sqrt{\frac{\log M_0}{N_0}} \right) +O(c_h)\\
&\leq& O\left(c_h+\|h\|_\infty\sqrt{\frac{n\log N_0}{N_0}}\right)\\
\end{eqnarray*}
The conclusion follows by applying Markov inequality.

\subsection{Proof of Corollary \ref{thm:thm2}}
 
Let $E^{\perp} \in \mathbb{R}^N$ be the space of vectors whose mean is zero on each block. This space has codimension $k$. Now, for any vector $x\in E^{\perp}$, we see that $x^{t}Ax = 0$, where $A$ is the matrix from Theorem \ref{thm:thm1}. As a result, the quadratic form $\Phi_h(X)$ is at most $O\left(c_h+  \|h\|_\infty\sqrt{\frac{n\log N_0}{N_0}}\right)I$ on  $E^{\perp}$ with arbitrarily high probability, implying that $\Phi_h(X)$ has at least $(N-k)$ eigenvalues that are at most $O\left(c_h+  \|h\|_\infty\sqrt{\frac{n\log N_0}{N_0}}\right)$. Applying the same argument to $-\Phi_h(X)$, the result follows.

\section{Proofs for Section \ref{secker}}
We start with the proof of Proposition $\ref{prop:space}$. We want to show that for a positive kernel, the space spanned by the $k$ top eigenvectors of $\Phi_h(X)$ is close to the space of piecewise constant functions $E$. We first observe that for a large enough number of samples, matrix $G_h$ is close to its finite sample version $\widehat{G_h}$, whose $ij$ entry is the average of the kernel over $X_i\times X_j$:

\begin{lemma}
For any $c>0$, we have:
\begin{eqnarray*}
P \left( ||G_h - \widehat{G_h}||\geq c\right) &\leq& 1-O\left(N_0\exp\left(-N_0\Omega \left(\min\left(\frac{c}{\|h'\|_\infty},\frac{c^2}{\|h'\|_\infty^2}\right)\right)\right)\right) 
\end{eqnarray*}
\end{lemma}
\begin{proof}
The desired operator norm can be bounded using entries magnitude as follows:
\begin{eqnarray}{\label{entry}}
P \left( ||G_h - \widehat{G_h}||\geq c\right) &\leq& P \left( ||G_h - \widehat{G_h}||^2_2\geq c^2\right) \nonumber\\
&\leq& \max_{ij} P \left( |G_h(i,j) - \widehat{G_h}(i,j)|\geq c/k\right) 
\end{eqnarray}
In order to control the error on entry $ij$, we write:
\begin{align*}
G_h(i,j) - \widehat{G_h}(i,j) \;=&\; \frac{1}{N_iN_j} \sum_{x\in X_i,y\in X_j} h(||x-y||) - \int h(||x-y||) d\mu_i(x)d\mu_j(y)\\
\;=&\; \frac{1}{N_i} \sum_{x\in X_i} \frac{1}{N_j} \sum_{y\in X_j} \left( h(||x-y||) - \int h(||x-y||) d\mu_j(y)\right)\\
\quad & + \frac{1}{N_i} \sum_{x\in X_i}  \left( \int h(||x-y||) d\mu_j(y) - \int h(||x-y||) d\mu_i(x)d\mu_j(y)\right)\\
\end{align*}

Since $\|h'\|_\infty$ is the Lipschitz constant of $||h(x-.)||$, we see by concentration that for fixed $x$ and for $y$ distributed according to $\mu_j$:
$$\| h(||x-y||) - \int h(||x-y||) d\mu_j(y)\|_{\psi_1} = O(\|h'\|_\infty)$$ 
where for a random variable $U$, $\|U\|_{\psi_1} = \sup_{p\geq 1} p^{-1} \left(E\|U\|^p\right)^{1/p}$ is its Orlicz $\psi_1$ norm.
As a consequence, conditionally to $N_j$, this implies (Corollary 5.17 in \cite{vershynin:nonasrmt}) that for any $\varepsilon>0$:
$$
P\left(|S_x| \geq \varepsilon\right) \leq 2\exp\left(-N_j\Omega \left(\min\left(\frac{\varepsilon}{\|h'\|_\infty},\frac{\varepsilon^2}{\|h'\|_\infty^2}\right)\right)\right)$$
with
$$
S_x = \frac{1}{N_j} \sum_{y\in X_j} \left( h(||x-y||) - \int h(||x-y||)\right)
$$
Hence by the union bound:
\begin{eqnarray*}
P\left(\left|\frac{1}{N_i} \sum_{x\in X_i} S_x\right|\geq\varepsilon\right)&\leq&  2N_i\exp\left(-N_j\Omega \left(\min\left(\frac{\varepsilon}{\|h'\|_\infty},\frac{\varepsilon^2}{\|h'\|_\infty^2}\right)\right)\right)\\
&\leq&  O\left(N_0\exp\left(-N_0\Omega \left(\min\left(\frac{\varepsilon}{\|h'\|_\infty},\frac{\varepsilon^2}{\|h'\|_\infty^2}\right)\right)\right)\right)
\end{eqnarray*}
Similarly, as the Lipschitz constant of $\int h(||.-y||) d\mu_j(y)$ is at most $\|h'\|_\infty$ as well, we get:
$$
P\left(\left|U\right|\geq\varepsilon\right)\leq  2\exp\left(-N_0\Omega \left(\min\left(\frac{\varepsilon}{\|h'\|_\infty},\frac{\varepsilon^2}{\|h'\|_\infty^2}\right)\right)\right)$$
with
$$
U = \frac{1}{N_i} \sum_{x\in X_i}  \left( \int h(||x-y||) d\mu_j(y) - \int h(||x-y||) d\mu_i(x)d\mu_j(y)\right)
$$
The last two inequalities together with (\ref{entry}) imply the desired claim.
\end{proof}

Let now $\widehat{M_h}$ be the matrix obtained from $\widehat{G_h}$ by multiplying the $ij$ entry by $\sqrt{w_iw_j}$. Applying the above lemma with $c=c_h$, its smallest eigenvalue can be lower bounded as follows:
\begin{eqnarray*}
\lambda_1(\widehat{M_h}) = \Omega(\lambda_1(\widehat{G_h})) = \Omega(\lambda_1(G_h)-c_h)= \Omega(K c_h)
\end{eqnarray*}
with arbitrarily high probability, assuming $N_0 \gg \|h'\|^2_\infty/c_h^2$ and $K\geq 2$.

Now, note that $\widehat{M_h}$ is the matrix of the quadratic form $\Phi_h(X)$ restricted to $E$. More precisely, the indicator functions of the clusters, normalized to have unit $L_2$-norm, form an orthornormal basis of $E$, and writing that quadratic form in this basis gives $\widehat{M_h}$. Let $\lambda$ be the smallest eigenvalue of $\widehat{M_h}$. By the variational characterization of eigenvalues,  there exist at least $k$ eigenvalues of $\Phi_h(X)$ that are at least $\lambda$. 
Let $H$ denote the space spanned by the $k$-top eigenvectors of $\Phi_h(X)$, and let $L$ denote the space spanned by the remaining $N-k$. We show using a perturbation argument that the maximum of the principal angles between space $E$ and space $H$ is small.  

 Let $x \in E^{\perp}$ be a unit vector. We may write $x = \alpha x_L + \beta x_H$ 
 with $\alpha^2 + \beta^2 = 1$, and $x_L$ and $x_H$ are unit vectors belonging respectively to $L$ and $H$.
 Then:
 $$
x^{t}\Phi_h(X)  x = \alpha^{2} x_L^{t} \Phi_h(X) x_L + \beta^{2}x_H^{t} \Phi_h(X) x_H
 $$
Since $x \in E^{\perp}$, we have $x^{t} A x =0$, where $A$ is the matrix defined in Theorem \ref{thm:thm1}.  Hence by Theorem \ref{thm:thm1}, with arbitrarily high probability:
$$
x^{t} \Phi_h(X) x  \leq O(c_h)
$$
provided $$N_0\geq N_1=O\left(\frac{n\|h\|^2_\infty}{c_h^2}\log\left(\frac{n\|h\|^2_\infty}{c_h^2}\right)\right)$$
Also, by assumption:
$$
 x_H^t \Phi_h(X) x_H \geq \lambda \geq K\Omega (c_h)
$$
As a consequence: 
$$
 d(x,L) = \beta \leq O(1/\sqrt{K}) 
$$
That is, the maximum angle between the $(N-k)$-flats $E^\perp$ and $L$ is $O(1/\sqrt{K})$. Hence, so is the maximum angle between their orthogonals $E$ and $H$, which is the desired claim.

\subsection{Proof of Corollary \ref{cor:g} }

Matrix $G_h$ has entries
$$
G_h(i,j) = \mathbb{E} h(||x_i-x_j||)
$$
where $x_i$ are independent random variables with law $\mathcal{N}(\mu_i,\Sigma_i)$, where $\mu_i$ and $\Sigma_i$ are the means and covariances of the two Gaussians in the mixture.
\begin{lemma}
If $u$ is a centered Gaussian random variable with covariance $\Sigma$, then:
$$\mathbb{E}(h(||u||)) = \det\left(I + \frac{1}{\tau^2}\Sigma\right)^{-\frac{1}{2}}$$
\end{lemma}

\begin{proof} 
\begin{eqnarray*}
\mathbb{E}(h(||u||)) &=& \int \frac{1}{\sqrt{(2\pi)^n\det(\Sigma)}} \exp\left(-\frac{1}{2}x^t\left(\Sigma^{-1}+\frac{1}{\tau^2}I\right)x\right) dx\\
&=& \frac{\det\left(\left(\Sigma^{-1}+\frac{1}{\tau^2}I\right)^{-1}\right)^{1/2}}{\det(\Sigma)^{1/2}}\\
&=& \det\left(I + \frac{1}{\tau^2}\Sigma\right)^{-\frac{1}{2}}
\end{eqnarray*}
\end{proof}

By standard algebraic manipulations, shifting the center amounts to scaling the expectation by a certain factor:

\begin{lemma}
If $u$ is a Gaussian random variable with covariance $\Sigma$ and mean $\mu$, then:
$$\mathbb{E}(h(||u||)) = \exp\left(-\frac{1}{\tau^2}\mu^t(I-(I+\tau^2\Sigma^{-1})^{-1})\mu\right)\; \det\left(I + \frac{1}{\tau^2}\Sigma\right)^{-\frac{1}{2}}$$
\end{lemma}

In particular, letting $B_h$ be the $2\times 2$ matrix with entries
$$
B_h(i,j) = \mathbb{E}h(||y_i-y_j||)
$$
where $y_i$ are independent random variables with law $\mathcal{N}(0,\Sigma_i)$, we see that $G_h$ is obtained from $B_h$ by scaling the off diagonal entries by a factor $\lambda$ that satisfies
\begin{eqnarray*}
\lambda &\leq& \exp\left(-\frac{1-(1+\Omega(\tau^2))^{-1}}{\tau^2} ||\mu_1-\mu_2||^2\right)\\
&\geq& 1- \Theta\left(\frac{||\mu_1-\mu_2||^2}{n}\right)
\end{eqnarray*} 
Because $\det B_h$ is non negative and the entries of $B_h$ are $\Theta(1)$, we deduce that 
\begin{eqnarray*}
\det G_h &=& \det B_h + (B_h)_{12}^2 - \lambda^2 (B_h)_{12}^2 \\
&\geq& (1-\lambda^2) (B_h)_{12}^2\\
&\geq& \Theta\left(\frac{||\mu_1-\mu_2||^2}{n}\right)
\end{eqnarray*}
Now, the largest entries of $G_h$ are the same as for $B_h$, that is, $\Theta(1)$, which implies that the maximal eigenvalue of $G_h$ is $\Theta(1)$ as well. From this we see that:
$$
\lambda_1(G_h) = \Theta\left(\frac{||\mu_1-\mu_2||^2}{n}\right)
$$
To conclude, it suffices to check that for our choice of kernel and assumptions on the variance of the Gaussians, $c_h=\Theta(1/n)$.

\section{Proof of Theorem \ref{thm:thm2}}
We first show that constant functions are sent to nearly constant functions by the convolution operator with kernel $h$ from $L^2(\mathbb{R}^n,\mu_i)$ to $L^2(\mathbb{R}^n,\mu_j)$. 
\begin{lemma}\label{lmcst}
Let $f_i(x) = \int h(\|y-x\|)d\mu_i(y)$. If $\mu_i$ and $\mu_j$ are supported on the sphere $S$, even, and satisfy a Poincar\'e inequality, then:
$$
\mathrm{Var}_{\mu_j}f_i = O\left(nc_{\tilde{h}}'^2\right)
$$
\end{lemma}
\begin{proof}
The gradient of $f_i$ is as follows:
$$
\nabla f_i(x) = \int (x-y)\tilde{h}(\|x-y\|) d\mu_i(y)
$$
For $x\in S$, the gradient of the restriction of $f_i$ to $S$ is
\begin{equation}\label{gexp}
\nabla f_{i|S}(x) = P_{T_xS} \nabla f_i(x) =  - P_{T_xS} \left(\int y \tilde{h}(\|x-y\|) d\mu_i(y)\right)
\end{equation} 
Denote by $M:L^2(\mathbb{R}^n,\mu_i)\to L^2(\mathbb{R}^n,\mu_j)$ the operator defined by
$$
Mg(x) = \int g(y)\tilde{h}(\|x-y\|) d\mu_i(y)
$$
From the structure of blocks described in Theorem \ref{thm:thm1}, and letting the sample size go to infinity, we get that $\|M-M'\| = O(c_{\tilde{h}}')$, where
$$
M'g(x) = \int g(y)M_{xy}'d\mu_i(y)
$$
and
$$
M_{xy}' = \int \tilde{h}(\|y-z'\|)d\mu_i(z') + \int \tilde{h}(\|x-z\|)d\mu_j(z) - \int \tilde{h}(\|z-z'\|)d\mu_i(z)d\mu_j(z')
$$ 
Calling $y$ the coordinate vector of $S$, that is, the identity map of $S$, the above equation expresses $M'y$ as the sum of two terms $T_1$ and $T_2$. The first one is 
$$
T_1 = \int y \left(\int \tilde{h}(\|y-z'\|)d\mu_i(z')\right) d\mu_i(y)
$$
we see that as $\mu_i$ is even, $\int \tilde{h}(\|y-z'\|)d\mu_i(z')$ is an even function of $y$. Hence multiplying it by $y$ gives an odd function whose integral against $\mu_j$ must be be zero as $\mu_j$ is even as well. Hence $T_1$ vanishes. 
The second term $T_2$ is
$$
T_2 = \left(\int y d\mu_i(y)\right)\left(\int \tilde{h}(\|x-z\|)d\mu_j(z) - \int \tilde{h}(\|z-z'\|)d\mu_i(z)d\mu_j(z')\right)
$$
As $\mu_i$ is even, it has zero mean so $T_2$ cancels.
From (\ref{gexp}), the above discussion gives:
\begin{eqnarray*}
\|\nabla f_{i|S}\|^2 &\leq& \|My\|^2\\
&\leq& \|(M-M')y\|^2\\
&\leq& O\left(c_{\tilde{h}}'^2\|y\|^2\right)\\
&\leq& O\left(nc_{\tilde{h}}'^2\right)
\end{eqnarray*}

The desired claim follows using Poincar\'e inequality.
\end{proof}

\begin{lemma}\label{dtild}
Taking $h=h_t$, we have:
$$
\mathrm{Var}_{\mu_j}f_i \leq  O(t^2/n)
$$
assuming $\mu_i$ and $\mu_j$ are supported on $S$, have $O(1)$ means and $O(1)$-concentration.
\end{lemma}
\begin{proof}
For any $x,y$ in $S$ we can write:
\begin{eqnarray*}
h_t(\|x-y\|) &=& \mathrm{Re}\, \exp\left(\frac{it}{\sqrt{n}}<x,y>\right)
\end{eqnarray*}
Hence we can express $f_i$ using a Fourier transorm:
\begin{eqnarray*}
f_i(x) &=& \mathrm{Re}\,\left(\exp(2nit) \widehat{\mu_i}(-tx/\sqrt{n})\right)
\end{eqnarray*}
As a consequence, for any unit vector $u$:
\begin{eqnarray*}
|<\nabla f_i(x),u>| &\leq& \frac{t}{\sqrt{n}} |<\nabla \widehat{\mu_i}(-tx/\sqrt{n}),u>|\\
&\leq& \frac{t}{\sqrt{n}} |\widehat{\mu_i^u}(-tx/\sqrt{n})|\\
&\leq& \frac{t}{\sqrt{n}} O(\|\mu_i^u\|_1)\\
&\leq& O(t/\sqrt{n})
\end{eqnarray*}
where $\mu_i^u$ is $\mu_i$ multiplied by function $x\mapsto <x,u>$, 
the last line using the fact that $\mu_i$ has $O(1)$-concentration and $O(1)$ mean. Hence $f_i$ is $O(t/\sqrt{n})$-Lipschitz. The lemma follows since $\mu_j$ has $O(1)$-concentration.
\end{proof}

To prove the first part of Theorem \ref{thm:thm2}, using Theorem \ref{thm:thm1}, it is sufficient to show that with arbitrarily high probability $\|A-B\|=O(\sqrt{n}c_{\tilde{h}}' )$, $A$ being the matrix given by Theorem \ref{thm:thm1}. By definition of $A$, we see that the entries of $A-B$ in the $ij$ block are given by
\begin{eqnarray*}
(A-B)_{xy} &=& \frac{1}{N}\left(\int h(\|x-z\|)d\mu_j(z) - \int h(\|x-z\|)d\mu_j(z)d\mu_i(x)\right) \\
&&+ \frac{1}{N}\left(\int h(\|y-z'\|)d\mu_i(z') - \int h(\|y-z'\|)d\mu_i(z')d\mu_j(y)\right)\\
&=& \frac{1}{N}  \left(\left(f_j(x) - \int f_j(x)d\mu_i(x)\right) + \left(f_i(y) - \int f_i(y)d\mu_j(y)\right)\right)
\end{eqnarray*}
Hence by Lemma \ref{lmcst}, the entries of $A-B$ have, conditionally to $N$,  variance $O(nc_{\tilde{h}}'^2/N^2)$. In particular, $A-B$ has expected squared Frobenius norm at most $O(nc_{\tilde{h}}'^2)$. 
Bounding the operator norm by the Frobenius norm and applying Markov inequality proves the desired bound on $\|A-B\|$ and concludes the proof of the first part of the theorem. 
For the second part of Theorem \ref{thm:thm2}, the argument is the same except one uses the bound given in Lemma \ref{dtild} instead of Lemma \ref{lmcst}. Expliciting the constant $c_{h_t}'=O(t\log^3n/\sqrt{n})$ then gives the desired bound. 
\subsection{Proof of Corollary \ref{correigen} } 

Let $g(r)=h(r^2/\sqrt{n})$. Since $R=\sqrt{n}$, we have:

$$
c_g' = \sup_{r \geq \Theta(\sqrt{n}/\log n)} \left(\frac{ \log^2(n)|g''(r)|+|g'(r)|}{r} \right) + \|g'\|_\infty/{\sqrt{n}}
$$

Now $$g'(r) = \frac{2r}{\sqrt{n}} h'\left(\frac{r^2}{\sqrt{n}}\right)$$
hence $g'(r)/r = O(1/\sqrt{n})$. Also:
$$g''(r) = \frac{2}{\sqrt{n}} h'\left(\frac{r^2}{\sqrt{n}}\right) + \frac{4r^2}{n} h'\left(\frac{r^2}{\sqrt{n}}\right)$$
So $|g''(r)| \leq O(1/\sqrt{n} + r^2/n)$ and $c_g'= O(\log^3 n/\sqrt{n})$. 
Finally, since $$\tilde{g}(r) = \frac{2}{\sqrt{n}}
h'\left(\frac{r^2}{\sqrt{n}}\right)$$
we have that $\sqrt{n}c_{\tilde{g}}'$ can be bounded as above since $h$
has bounded derivatives up to third order. 

\section{Proof of Theorem \ref{prop:space2}}
Since the desired conclusions are unchanged by scaling the components by a constant factor, and as we assume their variance is $\Theta(n)$, we can assume that their variance is $n$.
Let $\widetilde{\mu_i}$ be the pushforwards of $\mu_i$ by the closest point projection on $S$. The following lemma is easily proved:

\begin{lemma}\label{proj}
Measure $\widetilde{\mu_i}$ has $O(1)$-concentration and mean $O(1)$. 
\end{lemma}
\begin{proof}
Let $f:S\to \mathbb{R}$ be a $1$-Lipschitz function. To prove that $\widetilde{\mu_i}$ has $O(1)$-concentration, we prove that for $X$ distributed according to $\widetilde{\mu_i}$, there exists a number $c$ such that $\|f(X)-c\|_{\psi_1}=O(1)$. The range of $f$ on $S$ is contained in an interval of length $2\sqrt{n}$. By shifting $f$ if necessary, we can assume that $\|f\|_\infty=O(\sqrt{n})$. We also assume $f$ is smooth, which is sufficient. Define 
\begin{eqnarray*}
g:\mathbb{R}^n &\to& \mathbb{R}\\
x &\mapsto& f\left(\frac{x}{\|x\|}\right) \;\;\mathrm{ if }\;\; \|x\| \geq \sqrt{n}/2\\
x &\mapsto& \frac{2\|x\|}{\sqrt{n}} f\left(\frac{x}{\|x\|}\right)\;\; \mathrm{ else } 
\end{eqnarray*}
We have:
\begin{eqnarray*}
\nabla g(x) &=& \frac{\sqrt{n}}{\|x\|} \nabla f \left(\frac{x}{\|x\|}\right) \;\;\mathrm{ if }\;\; \|x\| \geq \sqrt{n}/2\\
&=& \frac{2}{\sqrt{n}} \left(\frac{x}{\|x\|} f \left(\frac{x}{\|x\|}\right) + \sqrt{n} \nabla f \left(\frac{x}{\|x\|}\right)\right) \;\; \mathrm{ else } 
\end{eqnarray*}
As a consequence function $g$ is $O(1)$-Lipschitz, hence by concentration, for $Y$ distributed according to $\mu_i$, there exists a number $c$ such that by $\|g(Y)-c\|_{\psi_1}=O(1)$.
Letting now $\bar{f}:x\mapsto f(x/\|x\|)$, we have that $P(g(Y)\neq\bar{f}(Y))\leq\exp(-\Theta(1)\sqrt{n})$ since $g$ and $\bar{f}$ only differ on $B(0,\sqrt{n}/2)$, which has exponentially small measure by concentration. Also clearly $\|g(Y)-\bar{f}(Y)\|_\infty\leq O(\sqrt{n})$. As a consequence, the $\psi_1$ norm of $g(Y)-\bar{f}(Y)$ is at most $O(\sqrt{n})$ times the $\psi_1$ norm of a Bernoulli variable with expectation $\exp(-\Theta(1)\sqrt{n})$. Since the $\psi_1$ norm of such variables is $O(1/\sqrt{n})$, $\|g(Y)-\bar{f}(Y)\|_{\psi_1}=O(1)$, from which we get $\|\bar{f}(Y)-c\|_{\psi_1}=O(1)$. This is what we wanted to prove, as $\bar{f}(Y)$ and $f(X)$ have the same distribution.

To relate the means of $\mu_i$ and $\widetilde{\mu_i}$, we notice that by concentration of the distance to the origin, the $1$-transportation distance between both measures is $O(1)$. In particular the means of $\mu_i$ and $\widetilde{\mu_i}$ differ by $O(1)$, hence the mean of $\widetilde{\mu_i}$ is $O(1)$. 
\end{proof}

The above lemma shows that we can apply Theorem \ref{thm:thm2} to the projected point cloud $\tilde{X}$: With arbitrarily high probability, matrix $\Phi_{h_t}(\tilde{X})$ is $\delta=O(t\log^3n/\sqrt{n})$ close to $B$ in the operator norm, assuming $N_0$ is $\Omega(\log(n/t)n^2/t^2)$. 

We would now like to argue that $B$ retains enough information about the components so that we can separate them. To do so, we restrict $B$ to the subspace $E_{u,v}$ of piecewise constant vectors supported on the two components $\tilde{X}_u$ and $\tilde{X}_v$, for some indices $u$ and $v$. In the orthornormal basis formed by the normalized indicator vectors of the two components, the $ij$ entry ($i,j\in \{u,v\}$) matrix of this restriction is $(\bar{w}_i\bar{w}_j)^{-1/2}\widetilde{G}_{h_t}(i,j)$, $\widetilde{G}_{h_t}$ being the $2\times 2$ matrix associated with $\widetilde{\mu}_u$ and $\widetilde{\mu}_v$, and $\bar{w}_i$ being the fraction of data points in the $i^{th}$ component. As the $\bar{w}_i$'s are $\Theta(1)$, the singular values of $B$ restricted to $V_{u,v}$ are within a constant factor of those of $\widetilde{G}_{h_t}$.

Now, using the power series expansion of $h_t$, one can show the following lower bound on the smallest singular value of the $2\times 2$ matrix $G_{h_t}$ associated with $\mu_u$ and $\mu_v$, based on the difference between their covariance matrices:

\begin{lemma}
There exists $C_1=\Theta(1)$ such that if $t \leq C_1\|\Sigma_u-\Sigma_v\|_2/\sqrt{n}$, the smallest singular value of $G_{h_t}$ is at least $\Omega(t^2\|\Sigma_u-\Sigma_v\|_2^2/n)$. Furthermore:
$$\|\widetilde{G}_{h_t} - G_{h_t}\|=O(t/\sqrt{n})$$
\end{lemma}

\begin{proof}

By Taylor's theorem, for $i,j\in\{u,v\}$, we have:
\begin{eqnarray*}
G_{h_t}(i,j) &=& \int \cos\left(\frac{t}{\sqrt{n}}<x,y>\right) d\mu_i(x)d\mu_j(y)\\
&=& \sum_{l=0}^\infty \int (-1)^l\frac{(t/\sqrt{n})^{2l}}{(2l)!} <x,y>^{2l}d\mu_i(x)d\mu_j(y)
\end{eqnarray*}
Let $x$ and $y$ be two independent random vectors distributed respectively according to $\mu_i$ and $\mu_j$.
Conditioned to $x=x_0\in\mathbb{R}^n$, $<x,y>$ has $O(\|x_0\|)$-concentration and mean $O(\|x_0\|)$, so its $\psi_1$ norm is $O(\|x_0\|)$. Hence $$\|<x,y>\|_{\psi_1} \leq O(\mathbb{E}\|x\|) \leq O(\sqrt{n})$$
 As a consequence the distribution of $|<x,y>/\sqrt{n}|$ decays exponentially. Hence its $l^{th}$ moment is controlled by the $l^{th}$ moment of an exponential distribution with mean $\Theta(1)$, that is, $\Theta(1)^ll!$. This implies
\begin{eqnarray*}
|G_{h_t}(i,j) - 1 + \int \frac{t^2}{2n} <x,y>^2d\mu_i(x)d\mu_j(y)| &\leq& \sum_{l=2}^\infty \frac{t^{2l}}{(2l)!}\Theta(1)^{2l}(2l)!\\
&\leq& O(t^4)
\end{eqnarray*}
for $t$ less than some numerical constant. Now 
\begin{eqnarray*}
\int <x,y>^2d\mu_i(x)d\mu_j(y) &=& \int y^t\Sigma_iy\; d\mu_j(y)\\
&=& \int \mathrm{trace }\; \Sigma_i yy^t d\mu_j(y)\\
&=&  \mathrm{trace }\; \Sigma_i \Sigma_j
\end{eqnarray*}
We may thus expand the determinant of $G_{h_t}$ as follows:
\begin{eqnarray*}
\det G_{h_t} &=& G_{h_t}(u,u)G_{h_t}(v,v) - G_{h_t}(u,v)^2\\
&=& \left(1 - \frac{t^2}{2n} <\Sigma_u,\Sigma_u> + O(t^4)\right)\left(1 - \frac{t^2}{2n} <\Sigma_v,\Sigma_v> + O(t^4)\right)\\
&& - \left(1 - \frac{t^2}{2n} <\Sigma_u,\Sigma_v> + O(t^4)\right)^2 \\
&=& -\frac{t^2}{2n} \|\Sigma_u-\Sigma_v\|_2^2 + O(t^4)
\end{eqnarray*}
Hence by assumption, for well chosen $C_1$, the first term in the expansion above dominates, so $|\det G_{h_t}|$ satisfies the desired lower bound. Since the entries of $G_{h_t}$ have absolute value less than $1$, the lower bound also holds for the smallest singular value of $G_{h_t}$.

To relate matrices $G_{h_t}$ and $\widetilde{G}_{h_t}$, we let $\delta x$ (resp. $\delta y$) be the difference between $x$ (resp. $y$) and its projection on $S$, so that $x-\delta x$ (resp. $y-\delta y$) is distributed according to $\widetilde{\mu_i}$ (resp. $\widetilde{\mu_j}$). We can write

$$
\widetilde{G}_{h_t}(i,j) = \int \cos\left(\frac{t}{\sqrt{n}}<x-\delta x,y-\delta y>\right) d\mu_i(x)d\mu_j(y)
$$

Also

$$
<x-\delta x,y-\delta y> = <x,y> - <\delta x,y> - <\delta y, x> + <\delta x, \delta y>
$$

By concentration and since $\mu_j$ has $O(1)$ mean, $|<\delta x,y>|$ has expectation $O(\|\delta x\|)$ conditioned to $\delta x$. Since $\mathbb{E}\|\delta x\|=O(1)$ by concentration of the distance to the origin, we have $\mathbb{E}|<\delta x,y>|=O(1)$. The last two terms above can be dealt with similarly, yielding that the distributions of $<x-\delta x,y-\delta y>$ and of $<x,y>$ are at $1$-transportation distance $O(1)$. Since $\cos(t./\sqrt{n})$ is $O(t/\sqrt{n})$-Lipschitz, we see that 
$$
|\widetilde{G}_{h_t}(i,j) - G_{h_t}(i,j)|= O(t/\sqrt{n})
$$
which concludes the proof.
\end{proof}

In particular, choosing $t = C_1\min_{u\neq v}\|\Sigma_u-\Sigma_v\|_2/\sqrt{n}  = C_1\Delta$, we see that for any $u\neq v$, the smallest singular value of $B$ restricted to $E_{u,v}$ is at least $\Omega(\Delta^4-O(\Delta/\sqrt{n}))$, which by assumption on $\Delta$ is also $\Omega(\Delta^4)$.

\begin{lemma}\label{spB}
For sufficiently small $C_2=\Theta(1)$, the columns of $f_{C_2\Delta^4}(B)$ with indices $i$ and $j$ are equal if $i$ and $j$ belong to the same component. If $i$ and $j$ belong to different components, their distance is $\Omega(\Delta^4/\sqrt{N})$.
\end{lemma}
\begin{proof}
Eigenvectors of $B$ with non zero eigenvalue are piecewise constant, so the first part is clear. Assume indices $i$ and $j$ respectively belong to distinct components $u$ and $v$. The distance between their columns is $\|f_{C_2\Delta^4}(B)e_{uv}\|$, where $e_{uv}$ has entries $1/\sharp X_u$ (resp. $-1/\sharp X_v$) at indices corresponding to component $u$ (resp. $v$), and $0$ else.

Vector $e_{uv}$ is in $E_{uv}$ and has norm $\Theta(1/\sqrt{N})$. From the singular value lower bound, there must exist a unit vector $x$ such that $|<e_{uv},Bx>|=\Omega(\Delta^4/\sqrt{N})$.  Denote by $E_{2C_2\Delta^4}$ the vector space generated by the singular vectors of $B$ with singular values at least $2C_2\Delta^4$, and write $x = \alpha y+\beta z$, where $y$ and $z$ are unit vectors respectively lying in $E_{2C_2\Delta^4}$ and in $E_{2C_2\Delta^4}^\perp$, and $\alpha^2+\beta^2=1$.  We have
\begin{eqnarray*}
|<e_{uv},Bx>| &=& |\alpha <e_{uv},By> + \beta <e_{uv},Bz>|\\
&=& O(|<e_{uv},By>|) + O(C_2\Delta^4/\sqrt{N})\\
&\leq & \max(C_3 |<e_{uv},By>|,C_4C_2 \Delta^4/\sqrt{N})
\end{eqnarray*}
for some constant $C_3$ and $C_4$. For small enough $C_2=\Theta(1)$, we will have $C_4C_2 \Delta^4/\sqrt{N} < |<e_{uv},Bx>|$, implying $$|<e_{uv},By>| \geq |<e_{uv},Bx>|/C_3 = \Omega(\Delta^4/\sqrt{N})$$ Now because $y\in E_{2C_2\Delta^4}$, as $f_{C_2\Delta^4}$ modifies eigenvalues by a factor at most $2$ in that range, there exists a matrix $F$ with the same eigenvectors as $B$, and with singular values between $1/2$ and $2$, such that $FBy=f_{C_2\Delta^4}(B)y$. Hence 
\begin{eqnarray*}
|<f_{C_2\Delta^4}(B)F^{-1}e_{uv},y>| &=&  |<F^{-1}e_{uv},f_{C_2\Delta^4}(B) y>|= |<F^{-1}e_{uv},FBy>|\\
&=& |<e_{uv},By>| = \Omega(\Delta^4/\sqrt{N})
\end{eqnarray*}
In particular $f_{C_2\Delta^4}(B)F^{-1}e_{uv}$ has norm at least $
\Omega(\Delta^4/\sqrt{N})$. But that vector equals $F^{-1}f_{C_2\Delta^4}(B)e_{uv}$, and as $F^{-1}$ doesn't change distances by more than a factor of $2$, we see that $\|f_{C_2\Delta^4}(B)e_{uv}\|=\Omega(\Delta^4/\sqrt{N})$, as claimed.
\end{proof}

Now, as $f_{C_2\Delta^4}$ is $1$-Lipschitz, the perturbation inequality proved in \cite{farf} states that
\begin{eqnarray*}
\|f_{C_2\Delta^4}(B) - f_{C_2\Delta^4}(\Phi_{h_t}(\tilde{X}))\|&\leq& O\left( \log\frac{\|B\|+\|\Phi_{h_t}(\tilde{X})\|}{\|\Phi_{h_t}(\tilde{X})-B\|}+2\right)^2 \|\Phi_{h_t}(\tilde{X})-B\|\\
&\leq& O(\delta \log^2\delta) \\
\end{eqnarray*}
Our assumption on $\Delta$ is chosen so that $\Delta^4/(\delta
\log^2\delta) = \Omega(K^3)$. Hence we may assume that
$\|\Phi_{h_t}(\tilde{X})-B\|=\delta<C_2\Delta^4$. By Weyl's theorem on
eigenvalue perturbations, $\Phi_{h_t}(\tilde{X})$ thus has at most
$k=\Theta(1)$ singular values larger than  $C_2\Delta^4$. Hence $ f_{C_2\Delta^4}(\Phi_{h_t}(\tilde{X}))$ has at most $k=\Theta(1)$ non zero eigenvalues. As a result $$\|f_{C_2\Delta^4}(B) - f_{C_2\Delta^4}(\Phi_{h_t}(\tilde{X}))\|^2_2 \leq  O(\delta^2 \log^4\delta) = O(\Delta^8/K^6)$$ This means that within each component, the expected square distance between a random column of $f_{C_2\Delta^4}(\Phi_{h_t}(\tilde{X}))$  and the column of $f_{C_2\Delta^4}(B)$ associated with that component is at most $O(\Delta^8/(NK^6))$. By Lemma \ref{spB}, this implies that after mapping data points to columns of $f_{C_2\Delta^4}(\Phi_{h_t}(\tilde{X}))$, the ratio between the maximum variance of the components and the minimum squared distance between their centers in an optimal solution to the k-means problem is $O(K^{-6})$. Applying any constant factor approximation algorithm for the $k$-means problem will thus cluster the data with the claimed error rate.

\section*{Acknowledgements} The authors would like to thank Jean-Daniel Boissonnat for his help throughout the elaboration of this paper. This work is partially supported by the Advanced Grant of the European Research Council GUDHI (Geometric Understanding in Higher Dimensions).

\bibliography{bibliographyMDPs15.bib}


\end{document}